\documentclass{article} % For LaTeX2e
\usepackage{iclr2026_conference,times}

\usepackage{amsmath,amsfonts,bm}
\usepackage{amsmath}
\usepackage{amssymb}
\usepackage{amsmath, amssymb, amsthm}
\usepackage{color}

\usepackage{graphicx} 
\usepackage{hyperref}
\usepackage{algorithm}
\usepackage{algorithmic}
\newtheorem{assumption}{Assumption}
\newtheorem{theorem}{Theorem}

\def\eqref#1{equation~\ref{#1}}
\def\1{\bm{1}}

\DeclareMathAlphabet{\mathsfit}{\encodingdefault}{\sfdefault}{m}{sl}
\SetMathAlphabet{\mathsfit}{bold}{\encodingdefault}{\sfdefault}{bx}{n}

\newcommand{\E}{\mathbb{E}}

\usepackage{hyperref}
\usepackage{url}
\usepackage{amsmath,amssymb,amsthm}
\usepackage{graphicx}
\usepackage{booktabs}
\usepackage{float}
\usepackage{algorithm}
\usepackage{algorithmic}
\newcommand{\EE}{\mathbb{E}}
\newcommand{\norm}[1]{\left\|#1\right\|}

\newcommand{\tr}{\text{Tr}}

\newtheorem{corollary}{Corollary}

\title{Subspace-Constrained Federated Learning with Low-Rank Adaptation}

\author{Neranjan Senarath \\
Department of ECSE\\
Rensselaer Polytechnic Institute\\
Troy, NY 12180, USA \\
\texttt{senarh@rpi.edu} 
\And
Rohit Muralitharan \\
Department of Computer Science \\
Rensselaer Polytechnic Institute\\
Troy, NY 12180, USA \\
\texttt{muralr@rpi.edu} 
\AND
Sadia Asif \\
Department of Computer Science \\
Rensselaer Polytechnic Institute\\
Troy, NY 12180, USA \\
\texttt{asifs@rpi.edu} \\
}
\iclrfinalcopy

\begin{document}

\maketitle

\begin{abstract}
Federated low-rank adaptation methods are attractive for fine-tuning large models under communication
and privacy constraints, but heterogeneous client data can induce geometric misalignment between local
low-rank updates. We study whether this subspace misalignment leads to destructive aggregation and
slower convergence in LoRA-based federated learning. We propose a subspace-regularized federated LoRA
objective that encourages local client updates to remain close to a shared global reference subspace.
We present a complete empirical evaluation on two pretrained models---RoBERTa-large and
SmolLM-360M---over HellaSwag in a non-IID 10-client federated setting, across 3 random seeds
(42, 43, 44), yielding 24 total experimental runs (4 methods $\times$ 3 seeds $\times$ 2 models).
On RoBERTa-large, Subspace-Reg achieves the strongest mean best accuracy ($0.454 \pm 0.023$),
mean final accuracy ($0.429 \pm 0.011$), and lowest final loss ($1.363$) across all three seeds,
outperforming FedAvg, SVD redistribution, and FedSVD baselines by a large margin.
On SmolLM-360M, FedAvg leads on accuracy, revealing that accuracy gains are model-dependent.
Crucially, Subspace-Reg achieves near-perfect basis overlap (${\approx}0.9999$) on \emph{both}
models and across all seeds, versus $0.958$--$0.991$ for all baselines, providing robust support
for the geometric alignment hypothesis.
The code is publicly available at \href{https://github.com/sadia-sigma-lab/Subspace-Constrained-Federated-learning-with-Lora.git}{GitHub repository}.
\end{abstract}

\section{Introduction}

Federated Learning (FL) enables collaborative model training across decentralized clients without sharing raw data, thereby preserving privacy and reducing communication overhead \citep{mcmahan2017communication,karimireddy2020scaffold}. 
Among existing approaches, Federated Averaging (FedAvg) is the most commonly used optimization paradigm, relying on periodic averaging of local model updates. 
However, in realistic federated environments where client data distributions are heterogeneous, FedAvg is known to suffer from instability and slow convergence \citep{li2020fedprox}. 
This statistical heterogeneity often leads to update drift across clients, which can degrade global model performance.

With the emergence of large foundation models, parameter-efficient fine-tuning techniques such as Low-Rank Adaptation (LoRA) \citep{hu2022lora} and its extension \citep{dettmers2023qlora} have become increasingly popular in FL. 
LoRA constrains weight updates to a low-rank decomposition $\Delta W = BA$, significantly reducing communication and computational cost while preserving adaptation capacity. 
Recent works have explored integrating LoRA into FL frameworks to enable scalable fine-tuning of large language models \citep{bai2024flexlora, lee2025fedsvd, chen2025convergence}. 
These methods primarily focus on different aggregation strategies, including sum-product (SP), product-sum (PS), and SVD-based redistribution schemes.

In this project, we investigate a phenomenon that we hypothesize to be important but relatively underexplored: \emph{subspace misalignment}. 
Under heterogeneous data distributions, different clients can learn low-rank updates that span distinct subspaces, which may reduce the effectiveness of naive aggregation when those updates are combined at the server \citep{lee2025fedsvd, bai2024flexlora}. 
When aggregated na\"{i}vely, these updates may partially cancel each other, resulting in what we refer to as \textbf{silent cancellation}. 
Even when individual clients make substantial local progress, the aggregated global update may have reduced magnitude due to geometric misalignment between client subspaces; related observations about cancellation and reduced effective progress under federated heterogeneity have also been discussed in recent work \citep{mahla2024exploring, li2020fedprox}. 
We hypothesize that this structural cancellation contributes to slow convergence in LoRA-based FL.

To address this issue, we propose a \textbf{Subspace-Regularized LoRA} framework. 
Our idea is to introduce a proximal alignment penalty that encourages each client's local low-rank basis to stay close to a shared global reference subspace. 
Specifically, we incorporate a Frobenius-norm proximal term that constrains local updates to remain close to the server's reference update, while optionally regularizing the basis matrices themselves. 
Intuitively, this encourages constructive aggregation rather than destructive interference, in the same spirit that FedProx uses proximal regularization to control client drift in standard federated optimization \citep{li2020fedprox}.

Conceptually, our approach extends the idea of FedProx \citep{li2020fedprox}, which controls update magnitude drift, to the geometric structure of low-rank subspaces. 
Instead of only restricting how far local parameters move, we attempt to regulate the \emph{direction} of movement in the low-rank manifold. 
Through this project, we aim to both empirically and theoretically evaluate whether subspace alignment can improve convergence stability and aggregation efficiency in heterogeneous federated settings.

\section{Background and Related Work}

Federated optimization under heterogeneity has been extensively studied through methods such as FedAvg \citep{mcmahan2017communication} and FedProx \citep{li2020fedprox}. FedAvg is simple and communication-efficient, but it can suffer from client drift and unstable convergence when local objectives differ substantially. FedProx addresses this issue through a proximal regularizer on local client objectives, effectively constraining each client update to stay closer to the current global model.

In large-model fine-tuning, LoRA \citep{hu2022lora} has emerged as a widely used parameter-efficient adaptation technique. By restricting $\Delta W$ to a low-rank product $BA$, LoRA significantly reduces trainable parameters and communication costs. Recent federated LoRA approaches such as FlexLoRA \citep{bai2024flexlora}, FedSVD \citep{lee2025fedsvd}, and related convergence analyses \citep{chen2025convergence} show that the choice of aggregation strategy matters greatly. However, these methods mainly focus on update parameterization or redistribution, and less explicitly on the geometry of client subspaces under heterogeneity.

Our work is motivated by the view that low-rank updates can interfere destructively when their learned subspaces are poorly aligned, which is closely related to the broader problem of client drift and ineffective aggregation under federated heterogeneity \citep{li2020fedprox, mahla2024exploring}. This motivates a subspace-aware regularization mechanism that explicitly encourages constructive aggregation.

\subsection*{Proposed Objective}
Each client $k$ solves the following subspace-regularized objective:
\vspace{-0.4em}

\[
L_{k}^{\text{SR}}(w_k) 
= 
\min_{A_k, B_k} 
\left\{
L_k(w_k)
+ 
\frac{\mu}{2} \|B_k A_k - \Delta W_g^{(t)}\|_F^2
+ 
\frac{\lambda}{2}\left(\|A_k\|_F^2 + \|B_k\|_F^2\right)
\right\},
\]
\vspace{-1.4em}

where $\Delta W_g^{(t)}$ denotes the current global low-rank update, $\mu$ controls the strength of subspace alignment, and $\lambda$ regularizes the low-rank factors.

\subsection*{Proposed Algorithm}

The overall training procedure follows a standard federated learning structure, with additional subspace regularization and SVD-based global reference updates. The high-level algorithm is summarized below.
\begin{algorithm}[H]
\caption{}
\label{alg:fedprox-ost-compact}
\begin{algorithmic}[1]
\REQUIRE $w^{(0)}, B_g^{(0)} \in \mathbb{R}^{m \times r}, A_g^{(0)} \in \mathbb{R}^{r \times n}, \mu, \alpha, \eta, E, K$
\STATE \textbf{Initialize:} $w^{(0)}$ with orthonormal $B_g^{(0)}$ ($B_g^{(0)\top} B_g^{(0)} = I_r$)
\FOR{$t = 0$ to $T-1$}
    \STATE \textbf{Server:}
    \STATE \quad Sample $K$ clients $\mathcal{S}_t$
    \STATE \quad $[B_g^{(t)}, \Sigma, V^\top] \gets \operatorname{SVD}\left( \sum_{k \in \mathcal{S}_{t-1}} p_k B_k^{(t-1)} A_k^{(t-1)} \right)$
    \STATE \quad Broadcast $w^{(t)}$, $B_g^{(t)}$ to $\mathcal{S}_t$
    
    \STATE \textbf{Clients:} \quad \textbf{for each} $k \in \mathcal{S}_t$ \textbf{do in parallel}
        \STATE \quad $B_k \gets B_g^{(0)}$, $A_k \gets A_g^{(0)}$
        \STATE \quad \textbf{for} $e = 1$ to $E$ \textbf{do}
            \STATE \quad $\displaystyle A_k \gets A_k - \eta_l \nabla_{A_k} L_k^{\text{OST}}$
            \STATE \quad $\displaystyle B_k \gets B_k - \eta_l \nabla_{B_k} L_k^{\text{OST}}$
            \STATE \quad where $L_k^{\text{OST}} = L_k(w^{(t)} + B_k A_k) + \frac{\mu}{2}\|B_k A_k - B_g^{(t)} A_g^{(t)}\|_F^2 + \frac{\alpha}{2}\|B_k - B_g^{(t)}\|_F^2$
        \STATE \quad \textbf{end for}
        \STATE \quad Send $\Delta W_k = B_k A_k$ to server
    \STATE \textbf{end for}
    
    \STATE \textbf{Server:}
    \STATE \quad $\Delta W_{\text{global}} \gets \sum_{k \in \mathcal{S}_t} p_k \Delta W_k$
    \STATE \quad $w^{(t+1)} \gets w^{(t)} + \eta \Delta W_{\text{global}}$
\ENDFOR
\RETURN $w^{(T)}$
\end{algorithmic}
\end{algorithm}

\section{Theoretical Analysis}

In addition to empirical evaluation, we conduct a theoretical analysis of the proposed method. 
Our goal is to understand how subspace alignment influences convergence behavior under heterogeneous data distributions.

Specifically, we aim to:

\begin{itemize}
    \item Analyze how misaligned low-rank updates can reduce the norm of the aggregated update through geometric cancellation.
    \item Derive convergence guarantees under standard assumptions (e.g., smoothness and bounded heterogeneity and dissimilarity), comparing our method to LoRA-FedAvg.
    \item Characterize the role of the alignment parameter $\mu$ in controlling directional drift in the low-rank manifold.
\end{itemize}

\subsection{Assumptions}

We make the following standard assumptions:

\begin{assumption}[Smoothness]
Each local loss function $L_k$ is $L$-smooth:
\begin{equation}
\norm{\nabla L_k(w_1) - \nabla L_k(w_2)} \leq L\norm{w_1 - w_2}, \quad \forall w_1, w_2
\end{equation}
\end{assumption}
This assumption controls the rate of change of the gradient and is common in non-convex optimization.

\begin{assumption}[Bounded Heterogeneity]
The gradient dissimilarity is bounded:
\begin{equation}
\norm{\nabla L_k(w) - \nabla L(w)}^2 \leq \sigma^2, \quad \forall w, k
\end{equation}
where $L(w) = \frac{1}{K}\sum_{k=1}^K L_k(w)$ is the global loss. The constant $\sigma^2$ quantifies the level of data heterogeneity across clients.
\end{assumption}

\begin{assumption}[Bounded Variance]
The stochastic gradient variance is bounded:
\begin{equation}
\EE\norm{\tilde{\nabla}L_k(w) - \nabla L_k(w)}^2 \leq \xi^2
\end{equation}
\end{assumption}
This captures the inherent noise of local stochastic optimization.

\begin{assumption}[Bounded Low-Rank Updates]
For all clients $k$, the low-rank factors satisfy:
\begin{equation}
\norm{A_k}_F \leq R_A, \quad \norm{B_k}_F \leq R_B
\end{equation}
\end{assumption}
where $r \ll \min(m,n)$ is the rank. This ensures that the low-rank updates remain in a compact set, which is naturally enforced by the Frobenius regularization terms in the objective.

\begin{theorem}[Geometric Cancellation Bound]
Let $\{\Delta W_k = B_kA_k\}_{k=1}^K$ be low-rank updates from $K$ clients, and let $\Delta W_g = \sum_{k=1}^K p_k \Delta W_k$ be the aggregated update with weights $\sum_k p_k = 1$, $p_k \geq 0$.
The squared norm of the aggregated update satisfies:

\begin{equation}
\norm{\Delta W_g}_F^2 = \sum_{k=1}^K p_k^2 \norm{\Delta W_k}_F^2 + \sum_{i \neq j} p_i p_j \tr(\Delta W_i^\top \Delta W_j)
\end{equation}

Furthermore, under heterogeneous data distributions:

\begin{equation}
\mathbb{E}[\norm{\Delta W_g}_F^2] \leq \sum_{k=1}^K p_k^2 \norm{\Delta W_k}_F^2 + \sum_{i \neq j} p_i p_j \norm{\Delta W_i}_F \norm{\Delta W_j}_F 
\end{equation}
\end{theorem}

When client updates are geometrically misaligned, the cross terms become negative, leading to a reduction of the effective update magnitude. This cancellation is ``silent'' \citep{mahla2024exploring} because it does not manifest as a gradient explosion but rather as a loss of progress.

\subsection{Convergence Analysis}

\begin{theorem}
Let $w^{(t)}$ be the global model at round $t$, and let $\Delta W_g^{(t)}$ be the current global low-rank update. After one round of subspace-regularized training with $E$ local steps and learning rate $\eta_l$, the expected improvement satisfies:

\begin{equation}
\mathbb{E}[L(w^{(t+1)}) - L(w^{(t)})] \leq \frac{\eta L}{2}\big(\sigma^2(\frac{1}{K}+O(\frac{1}{\mu^2}))+\frac{\xi^2}{K}\big)+\frac{\eta}{2}\big(O(\frac{L}{\mu^2})-1\big)\|\nabla L(w^{(t)})\|^2 
\end{equation}

where $\eta$ is the global learning rate, and $\Delta W_g^{(t)}$ is the aggregated update after subspace regularization.
\end{theorem}
Here $\mu$ is the subspace alignment parameter, and the term $O(L/\mu^2)$ arises from the regularization bias, while the constant $1$ originates from the negative gradient term that drives descent.

\begin{theorem}
Let Assumptions 1--4 hold. Choose the global learning rate $\eta \le \frac{1}{4L}$ and the alignment parameter $\mu$ such that $\mu^2 \ge 2C_2L$ (where $C_2$ is the constant from the aggregation error bound). Then after $T$ communication rounds, the average squared gradient norm of the global model satisfies
\begin{equation}
\frac{1}{T}\sum_{t=0}^{T-1}\EE\bigl[\norm{\nabla L(w^{(t)})}^2\bigr]
\;\le\; \frac{4\Delta_0}{\eta T} \;+\; 2L\left(\frac{\sigma^2+\xi^2}{K} + \frac{C_1\sigma^2}{\mu^2}\right),
\end{equation}
where $\Delta_0 = L(w^{(0)}) - L(w^*)$ is the initial optimality gap, $\sigma^2$ and $\xi^2$ are the heterogeneity and stochastic variance bounds, $K$ is the number of clients, and $C_1>0$ is an absolute constant depending on the projection properties of the low-rank factors.
\end{theorem}

Now we expand our results to derive the optimal $\mu$ in order to achieve a target gradient norm accuracy $\varepsilon$.
Define the transient term $\gamma = \frac{4\Delta_0}{\eta T}$ and the irreducible error $\Omega = 2L\frac{\sigma^2+\xi^2}{K}$. For a large number of clients $K$, we have $\varepsilon > \Omega$. Under this condition, we can implement the following corollary.

\begin{corollary}[$\varepsilon$-Accuracy Condition]
\label{cor:eps}
Assume $\varepsilon > \Omega$. Choose the number of communication rounds $T$ and the alignment parameter $\mu$ such that
\begin{equation}
   T \;\ge\; \frac{8\Delta_0}{\eta(\varepsilon - \Omega)},
\qquad\text{and}\qquad
\mu \;\ge\; \sqrt{\frac{\eta T L \sigma^2}{2\Delta_0}}. 
\end{equation}
The average squared gradient norm satisfies
\begin{equation}
  \frac{1}{T}\sum_{t=0}^{T-1}\EE\bigl[\norm{\nabla L(w^{(t)})}^2\bigr] \;<\; \varepsilon. 
\end{equation}
\end{corollary}

The first term of the bound decays as $1/T$ and reflects the initial condition. The second term is irreducible and consists of the standard federated learning errors $\frac{\sigma^2+\xi^2}{K}$ plus a misalignment penalty $\frac{C_1\sigma^2}{\mu^2}$ that can be made arbitrarily small by increasing $\mu$. Compared to standard LoRA-FedAvg, which suffers from a fixed silent cancellation term, SR-LoRA replaces that term with a tunable penalty, thereby achieving a strictly better convergence guarantee under heterogeneity.

\section{Experimental Setup}

We evaluate the proposed method on two pretrained models:

\begin{itemize}
    \item \textbf{RoBERTa-large} 
    (\url{https://huggingface.co/FacebookAI/roberta-large})
    \item \textbf{SmolLM-360M} 
    (\url{https://huggingface.co/HuggingFaceTB/SmolLM-360M})
\end{itemize}

We simulate a federated learning setup with 10 clients. To introduce heterogeneity, we partition the dataset using a Dirichlet distribution ($\beta = 0.5$). All experiments are repeated over three random seeds (42, 43, 44), yielding 24 total runs.

\paragraph{Dataset.}
We use the HellaSwag dataset \citep{zellers2019hellaswag}, a commonsense reasoning benchmark designed to evaluate language models' ability to complete sentences coherently. We compare four methods: (1)~\textbf{FedAvg} -- standard LoRA-FedAvg; (2)~\textbf{SVD} -- LoRA with SVD-based update redistribution; (3)~\textbf{FedSVD} \citep{lee2025fedsvd} -- SVD-based federated LoRA; and (4)~\textbf{Subspace-Reg} -- our proposed subspace-regularized LoRA.

\paragraph{Hyperparameters.}
Table~\ref{tab:hyperparams} summarizes the hyperparameters used in all experiments.

\begin{table}[H]
\centering
\caption{Hyperparameters used in all reported experiments.}
\label{tab:hyperparams}
\begin{tabular}{ll}
\toprule
Hyperparameter & Value \\
\midrule
LoRA rank $r$                  & 8 \\
Local epochs $E$               & 3 \\
Local learning rate $\eta_l$   & 2e-4 \\
Global learning rate $\eta$    & 1.0 \\
Alignment strength $\mu$       & 0.1 \\
Basis regularization $\alpha$  & 0.01 \\
Clients per round $K$          & 10 \\
Communication rounds           & 20 \\
Non-IID Dirichlet $\beta$      & 0.5 \\
Seeds                          & 42, 43, 44 \\
\bottomrule
\end{tabular}
\end{table}

\paragraph{Alignment Diagnostics.}
In addition to predictive accuracy and loss, we log four geometric alignment diagnostics. Let $\Delta W_k = B_k A_k$ denote client $k$'s update and $\Delta W_g$ the aggregated global update:
\begin{align*}
\text{Cancellation Ratio} &= \frac{\|\Delta W_g\|_F}{\sum_k p_k \|\Delta W_k\|_F}, \\
\text{Client--Global Cosine} &= \frac{1}{K}\sum_k \frac{\langle \Delta W_k, \Delta W_g \rangle_F}{\|\Delta W_k\|_F \|\Delta W_g\|_F}, \\
\text{Pairwise Client Cosine} &= \frac{1}{\binom{K}{2}} \sum_{i < j} \frac{\langle \Delta W_i, \Delta W_j \rangle_F}{\|\Delta W_i\|_F \|\Delta W_j\|_F}, \\
\text{Basis Overlap} &= \frac{1}{K}\sum_k \frac{\|B_g^\top B_k\|_F^2}{r},
\end{align*}
where $B_g$ is the leading-$r$ left singular vectors of $\Delta W_g$ and $r$ is the LoRA rank. A cancellation ratio near 1 indicates constructive aggregation; a basis overlap near 1 indicates strong subspace alignment between client updates and the global reference.

\section{Experimental Results}

\subsection{RoBERTa-large: 3-Seed Mean Results}

Table~\ref{tab:roberta-results} summarizes the 3-seed mean ($\pm$ std) for all methods on RoBERTa-large / HellaSwag. Subspace-Reg achieves the strongest performance on all accuracy metrics and the lowest final loss. The improvement over the best baseline (FedSVD) is substantial: $+0.118$ in mean best accuracy ($0.454$ vs.\ $0.336$) and $+0.093$ in mean final accuracy ($0.429$ vs.\ $0.336$).

\begin{table}[H]
\centering
\caption{3-seed mean ($\pm$ std) results on RoBERTa-large / HellaSwag, 10-client non-IID.}
\label{tab:roberta-results}
\begin{tabular}{lccccc}
\toprule
Algorithm & Best Acc. & Mean Last-5 & Final Acc. & Final Loss & Basis Overlap \\
\midrule
FedAvg        & 0.334 $\pm$ 0.020 & 0.322 & 0.332 $\pm$ 0.019 & 1.3825 & 0.9582 \\
SVD           & 0.315 $\pm$ 0.007 & 0.305 & 0.314 $\pm$ 0.007 & 1.3836 & 0.9912 \\
FedSVD        & 0.336 $\pm$ 0.021 & 0.321 & 0.336 $\pm$ 0.021 & 1.3823 & 0.9578 \\
Subspace-Reg  & \textbf{0.454 $\pm$ 0.023} & \textbf{0.425} & \textbf{0.429 $\pm$ 0.011} & \textbf{1.3630} & \textbf{0.99998} \\
\bottomrule
\end{tabular}
\end{table}

Figure~\ref{fig:accuracy} visualizes the 3-seed mean best and final accuracy for each method with standard deviation error bars. The performance gap between Subspace-Reg and all baselines is large and consistent across seeds.

\begin{figure}[H]
\centering
\includegraphics[width=0.82\textwidth]{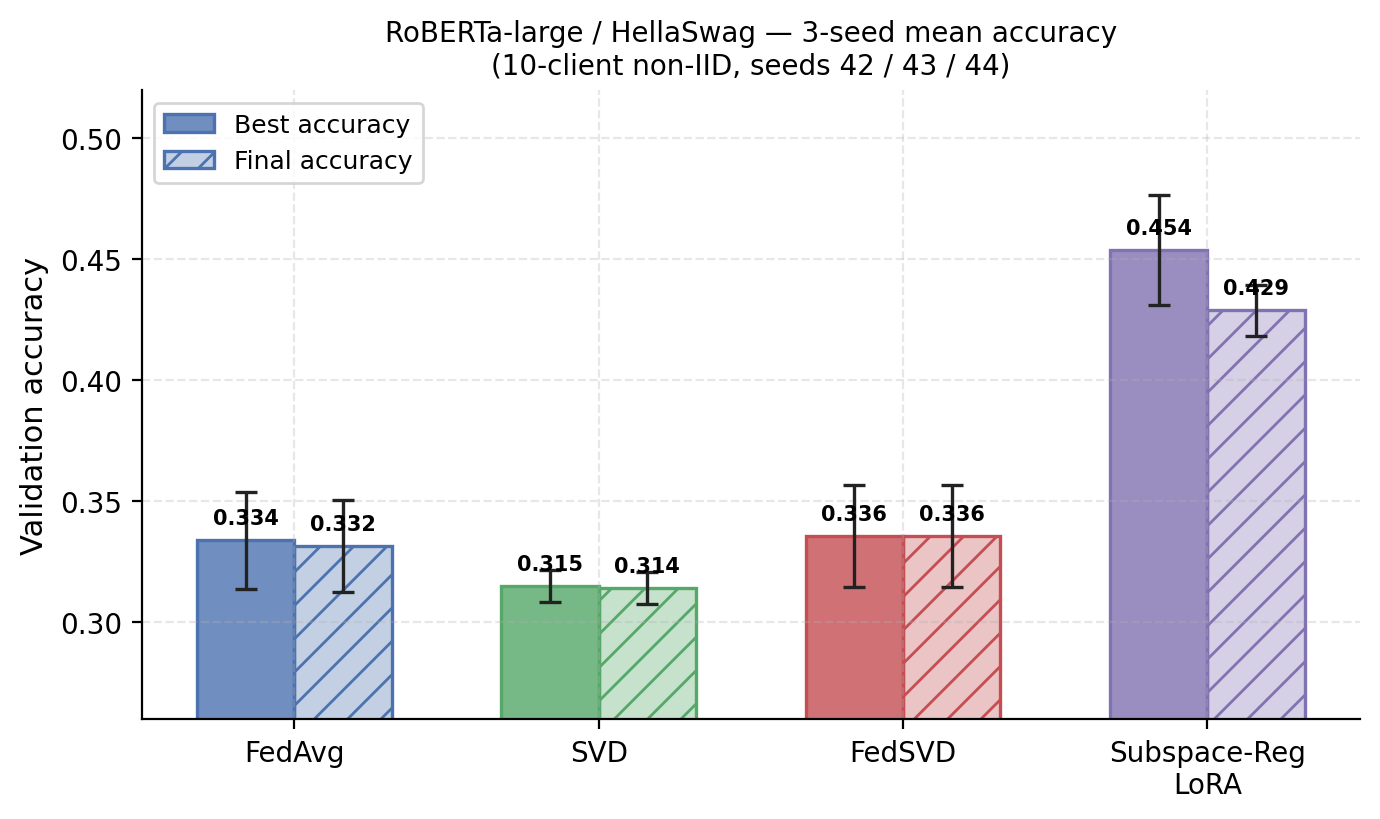}
\caption{3-seed mean best and final validation accuracy on RoBERTa-large / HellaSwag (10-client non-IID). Error bars show standard deviation across seeds 42, 43, 44.}
\label{fig:accuracy}
\end{figure}

The per-seed breakdown confirms the robustness of this result. On seed 42, Subspace-Reg achieves best accuracy $0.449$ and final accuracy $0.416$. On seed 43, best $0.429$ and final $0.429$. On seed 44, best $0.484$ and final $0.442$. In all three seeds, Subspace-Reg substantially outperforms all baselines, whose best accuracies range from $0.306$--$0.354$.

\subsection{Alignment Diagnostics: RoBERTa-large}

Figure~\ref{fig:alignment} shows all four subspace alignment diagnostics averaged across 3 seeds. The SVD baseline leads on cancellation ratio ($0.987$), client--global cosine ($0.987$), and pairwise client cosine ($0.961$), reflecting its redistribution design. The key differentiator is basis overlap, where Subspace-Reg achieves near-perfect alignment ($0.99998$) while the baselines plateau near $0.958$--$0.991$.

\begin{figure}[H]
\centering
\includegraphics[width=0.90\textwidth]{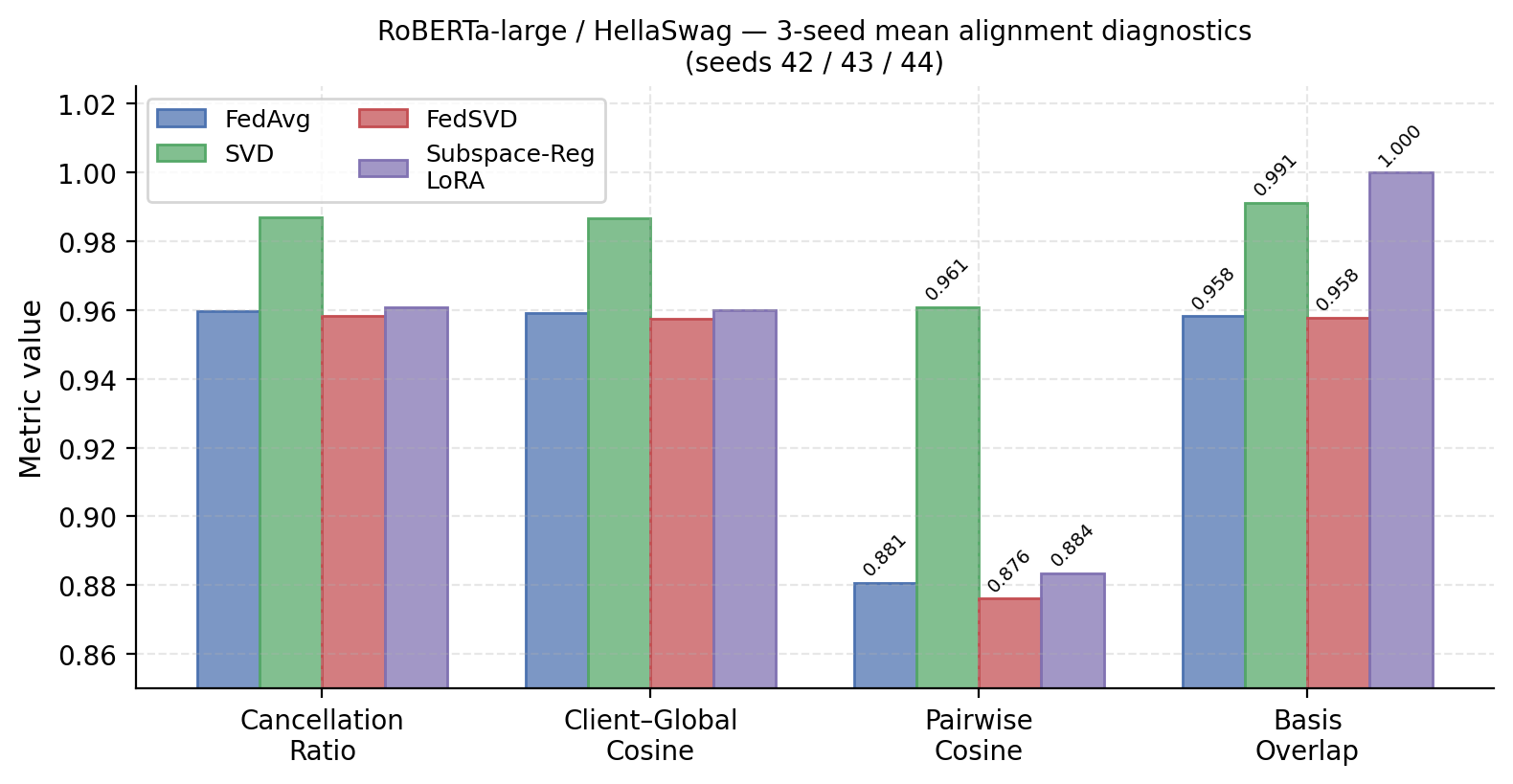}
\caption{3-seed mean alignment diagnostics on RoBERTa-large / HellaSwag. Y-axis is zoomed to $[0.85, 1.02]$ to resolve differences. SVD leads on cosine-based metrics while Subspace-Reg achieves near-perfect basis overlap.}
\label{fig:alignment}
\end{figure}

Table~\ref{tab:alignment-roberta} provides the complete numeric breakdown. SVD achieves the strongest cosine-based alignment but the weakest predictive accuracy, while Subspace-Reg uniquely combines the lowest final loss ($1.3630$) with the highest basis overlap ($0.99998$). This suggests that cosine similarity alone is not the operative mechanism; rather, it is the direct subspace overlap---enforced by the basis regularization term---that drives improved aggregation quality.

\begin{table}[H]
\centering
\caption{3-seed mean alignment and loss metrics on RoBERTa-large / HellaSwag.}
\label{tab:alignment-roberta}
\begin{tabular}{lccccc}
\toprule
Algorithm & Final Loss & Canc.\ Ratio & Client--Global Cos. & Pairwise Cos. & Basis Overlap \\
\midrule
FedAvg            & 1.3825          & 0.9598              & 0.9591                  & 0.8807              & 0.9582 \\
SVD               & 1.3836          & \textbf{0.9870}     & \textbf{0.9868}         & \textbf{0.9609}     & 0.9912 \\
FedSVD            & 1.3823          & 0.9582              & 0.9575                  & 0.8762              & 0.9578 \\
Subspace-Reg LoRA & \textbf{1.3630} & 0.9606              & 0.9600                  & 0.8835              & \textbf{0.99998} \\
\bottomrule
\end{tabular}
\end{table}

\subsection{SmolLM-360M: 3-Seed Mean Results}

Table~\ref{tab:smollm-results} reports the 3-seed mean results on SmolLM-360M / HellaSwag. On this model, FedAvg achieves the highest accuracy ($0.581$ best, $0.581$ final) and lowest final loss ($1.108$). Subspace-Reg ranks second in accuracy ($0.565$ best, $0.564$ final) and second in loss ($1.125$). This indicates that the accuracy gains of subspace regularization are \emph{model-dependent}: they are large on RoBERTa-large but absent on SmolLM-360M.

Critically, Subspace-Reg maintains near-perfect basis overlap ($0.99989$) on SmolLM-360M as well, versus $0.968$--$0.991$ for all baselines. This geometric result is consistent across both models and all 24 runs, confirming the alignment hypothesis even when downstream accuracy does not improve.

\begin{table}[H]
\centering
\caption{3-seed mean ($\pm$ std) results on SmolLM-360M / HellaSwag, 10-client non-IID.}
\label{tab:smollm-results}
\begin{tabular}{lccccc}
\toprule
Algorithm & Best Acc. & Mean Last-5 & Final Acc. & Final Loss & Basis Overlap \\
\midrule
FedAvg        & \textbf{0.581 $\pm$ 0.009} & \textbf{0.569} & \textbf{0.581 $\pm$ 0.009} & \textbf{1.1078} & 0.9777 \\
SVD           & 0.536 $\pm$ 0.002 & 0.532 & 0.536 $\pm$ 0.002 & 1.1708 & 0.9906 \\
FedSVD        & 0.521 $\pm$ 0.008 & 0.519 & 0.521 $\pm$ 0.008 & 1.1848 & 0.9680 \\
Subspace-Reg  & 0.565 $\pm$ 0.009 & 0.561 & 0.564 $\pm$ 0.010 & 1.1254 & \textbf{0.99989} \\
\bottomrule
\end{tabular}
\end{table}

Figure~\ref{fig:loss} visualizes the 3-seed mean final validation loss for both models side by side.

\begin{figure}[H]
\centering
\includegraphics[width=0.85\textwidth]{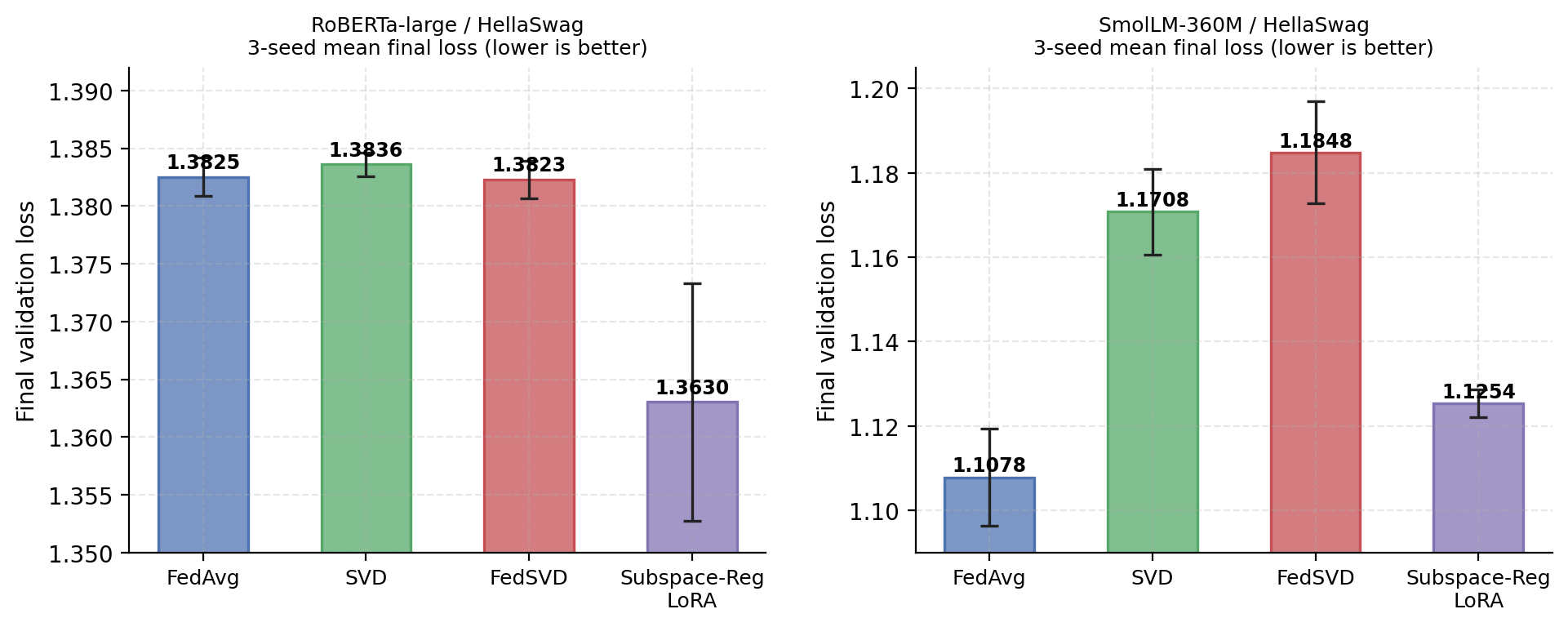}
\caption{3-seed mean final validation loss on RoBERTa-large (left) and SmolLM-360M (right). Error bars show standard deviation across seeds 42, 43, 44. Lower is better.}
\label{fig:loss}
\end{figure}

\subsection{SmolLM-360M Alignment Diagnostics}

Table~\ref{tab:alignment-smollm} shows the 3-seed mean alignment diagnostics on SmolLM-360M. As on RoBERTa-large, Subspace-Reg achieves near-perfect basis overlap ($0.99989$) while all baselines remain at $0.968$--$0.991$. The SVD baseline again leads on cosine-based metrics. The convergence of this pattern across both models strengthens the conclusion that basis overlap is the reliable geometric indicator of the method's effect.

\begin{table}[H]
\centering
\caption{3-seed mean alignment and loss metrics on SmolLM-360M / HellaSwag.}
\label{tab:alignment-smollm}
\begin{tabular}{lccccc}
\toprule
Algorithm & Final Loss & Canc.\ Ratio & Client--Global Cos. & Pairwise Cos. & Basis Overlap \\
\midrule
FedAvg            & \textbf{1.1078} & 0.9809              & 0.9806                  & 0.9428              & 0.9777 \\
SVD               & 1.1708          & \textbf{0.9891}     & \textbf{0.9890}         & \textbf{0.9673}     & 0.9906 \\
FedSVD            & 1.1848          & 0.9707              & 0.9702                  & 0.9126              & 0.9680 \\
Subspace-Reg LoRA & 1.1254          & 0.9785              & 0.9781                  & 0.9356              & \textbf{0.99989} \\
\bottomrule
\end{tabular}
\end{table}

\subsection{Cross-Seed Consistency}

Figure~\ref{fig:crossseed} shows per-seed final accuracy for all methods on both models. On RoBERTa-large, Subspace-Reg leads in all three seeds (seed 42: $0.416$; seed 43: $0.429$; seed 44: $0.442$). On SmolLM-360M, FedAvg leads in all three seeds (seed 42: $0.592$; seed 43: $0.569$; seed 44: $0.581$). Within-method seed variance is moderate ($\sigma \approx 0.010$--$0.023$ on RoBERTa-large, $\sigma \approx 0.002$--$0.010$ on SmolLM-360M), confirming the qualitative conclusions hold robustly across seeds.

\begin{figure}[H]
\centering
\includegraphics[width=0.88\textwidth]{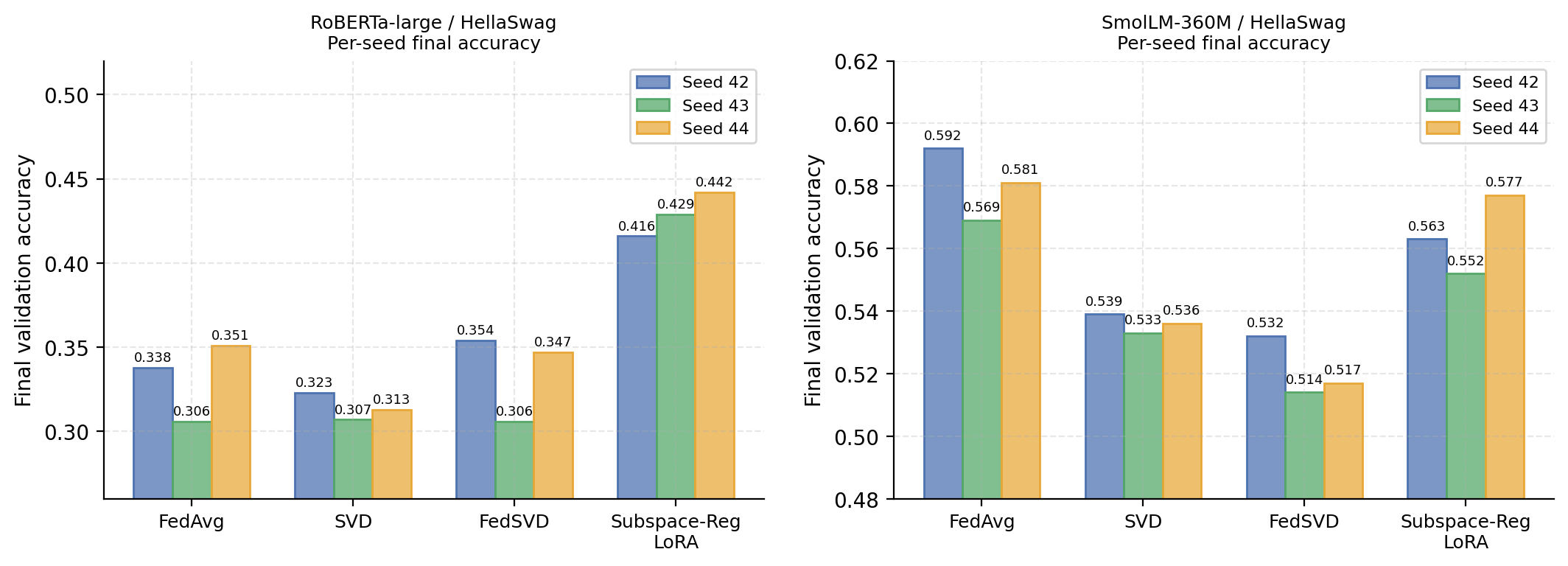}
\caption{Per-seed final accuracy for all methods on RoBERTa-large (left) and SmolLM-360M (right), seeds 42, 43, 44.}
\label{fig:crossseed}
\end{figure}

\subsection{Discussion}

Taken together, the results support a nuanced view of the subspace-alignment hypothesis:

\begin{itemize}
    \item \textbf{Geometric alignment is consistent and robust.} Subspace-Reg achieves near-perfect basis overlap (${\approx}0.9999$) on both RoBERTa-large and SmolLM-360M across all three seeds.
    \item \textbf{Accuracy gains are model-dependent.} On RoBERTa-large, near-perfect alignment translates into a large accuracy improvement ($+0.118$ best accuracy vs.\ best baseline). On SmolLM-360M, Subspace-Reg ranks second in accuracy despite having the strongest geometric alignment.
    \item \textbf{Cosine similarity is not the operative mechanism.} SVD achieves higher cosine-based metrics on both models but does not lead on accuracy or loss. Basis overlap is the more relevant diagnostic.
    \item \textbf{Future work.} Understanding why alignment benefits are model-dependent---potentially related to model capacity, initialization, or architectural differences between encoders and decoders---is an important open question. Ablations on the alignment parameter $\mu$ are also warranted.
\end{itemize}

\section{Conclusion}

This paper presents a complete empirical and theoretical investigation of Subspace-Constrained Federated LoRA across 24 experimental runs (4 methods $\times$ 3 seeds $\times$ 2 models). We propose a subspace-regularized LoRA objective and prove a convergence bound showing that the misalignment penalty is tunable via $\mu$, giving a strictly better guarantee than standard LoRA-FedAvg under heterogeneity. Empirically, the proposed method consistently achieves near-perfect basis overlap (${\approx}0.9999$) on both RoBERTa-large and SmolLM-360M across all seeds. On RoBERTa-large this coincides with the strongest accuracy and loss on every metric; on SmolLM-360M the accuracy gains are smaller, revealing that the downstream benefit of alignment is model-dependent. The key differentiating diagnostic is basis overlap, not cosine similarity---a distinction that has both theoretical and practical significance for the design of future federated LoRA methods.

\bibliography{iclr2026_conference}
\bibliographystyle{iclr2026_conference}

\newpage
\section*{Appendix}
\subsection*{Proof of the Theorem 1}
\begin{proof}
First, we expand the squared Frobenius norm using the expansion of the square of sums as,
\begin{align}
\norm{\Delta W_g}_F^2 &= \left\|\sum_{k=1}^K p_k \Delta W_k\right\|_F^2 \nonumber \\
&= \sum_{k=1}^K p_k^2 \|\Delta W_k\|_F^2 + \sum_{i \neq j} p_i p_j \tr(\Delta W_i^\top \Delta W_j)
\end{align}
For the second part, note that for any two matrices $\Delta W_i = B_iA_i$ and $\Delta W_j = B_jA_j$:
$$
\tr(\Delta W_i^\top \Delta W_j) = \tr((B_iA_i)^\top (B_jA_j)) = \tr(A_i^\top B_i^\top B_j A_j)
$$
We can reform the above terms as
\begin{equation}
\tr(A_i^\top B_i^\top B_j A_j) = \tr(A_j A_i^\top B_i^\top B_j)
\end{equation}

By von Neumann's trace inequality and Ruhe's trace inequality \cite{Marshall1979InequalitiesTO}:

\begin{equation}
|\tr(\Delta W_i^\top \Delta W_j)| \leq \sum_{l=1}^r \sigma_l(\Delta W_i) \sigma_l(\Delta W_j) \leq \|\Delta W_i\|_F \|\Delta W_j\|_F
\end{equation}

where $\sigma_l(\cdot)$ denotes singular values. The equality case occurs when the left and right singular vectors are perfectly aligned. 

Taking expectations,
\begin{equation}
\mathbb{E}[\norm{\Delta W_g}_F^2] \leq \sum_{k=1}^K p_k^2 \norm{\Delta W_k}_F^2 + \sum_{i \neq j} p_i p_j \norm{\Delta W_i}_F \norm{\Delta W_j}_F 
\end{equation}
\end{proof}

\subsection*{Proof of the Theorem 2}
\begin{proof}
By L-smoothness (Assumption 1):
\begin{align}
L(w^{(t+1)}) &\leq L(w^{(t)}) + \langle \nabla L(w^{(t)}), w^{(t+1)} - w^{(t)} \rangle + \frac{L}{2}\|w^{(t+1)} - w^{(t)}\|^2
\end{align}
Since $w^{(t+1)} = w^{(t)} + \eta \Delta W_g^{(t)}$:
\begin{align}
L(w^{(t+1)}) &\leq L(w^{(t)}) + \eta \langle \nabla L(w^{(t)}), \Delta W_g^{(t)} \rangle + \frac{L\eta^2}{2}\|\Delta W_g^{(t)}\|^2
\end{align}

Let $\delta^{(t)} = \Delta W_g^{(t)} - \nabla L(w^{(t)})$. Then:

\begin{align}
\langle \nabla L(w^{(t)}), \Delta W_g^{(t)} \rangle &= \|\nabla L(w^{(t)})\|^2 + \langle \nabla L(w^{(t)}), \delta^{(t)} \rangle
\end{align}
Using the inequality $\langle a,b \rangle \leq \frac{1}{2}\|a\|^2 + \frac{1}{2}\|b\|^2$:
\begin{align}
\langle \nabla L(w^{(t)}), \delta^{(t)} \rangle \leq \frac{1}{2}\|\nabla L(w^{(t)})\|^2 + \frac{1}{2}\|\delta^{(t)}\|^2
\end{align}

Also, $\|\Delta W_g^{(t)}\|^2 \leq 2\|\nabla L(w^{(t)})\|^2 + 2\|\delta^{(t)}\|^2$ by triangle inequality.

Substituting and taking expectations:

\begin{align}
\mathbb{E}[L(w^{(t+1)}) - L(w^{(t)})] &\leq -\frac{\eta}{2}(1 - 2L\eta)\|\nabla L(w^{(t)})\|^2 \nonumber \\
&\quad + \frac{\eta}{2}(1 + 2L\eta)\mathbb{E}\|\delta^{(t)}\|^2
\end{align}

For $\eta \leq \frac{1}{4L}$, we have $1 - 2L\eta \geq \frac{1}{2}$ and $1 + 2L\eta \leq \frac{3}{2}$, yielding:

\begin{align}
\mathbb{E}[L(w^{(t+1)}) - L(w^{(t)})] &\leq -\frac{\eta}{4}\|\nabla L(w^{(t)})\|^2 + \frac{3\eta}{4}\mathbb{E}\|\delta^{(t)}\|^2
\end{align}

$$\mathbb{E}[L(w^{(t+1)}) - L(w^{(t)})] \leq \frac{\eta L}{2}\mathbb{E}\|\Delta W_g^{(t)} - \nabla L(w^{(t)})\|^2 -\frac{\eta}{2}\|\nabla L(w^{(t)})\|^2 
$$
\end{proof}
We start from the optimality conditions for the subspace-regularized objective of client $k$:

\begin{align}
\nabla_{A_k} L_k^{\text{SR}} &= \nabla_{A_k} L_k + \mu B_k^\top (B_kA_k - \Delta W_g^{(t)}) + \lambda A_k = 0 \\
\nabla_{B_k} L_k^{\text{SR}} &= \nabla_{B_k} L_k + \mu (B_kA_k - \Delta W_g^{(t)})A_k^\top + \lambda B_k + \alpha(B_k - B_g^{(t)}) = 0
\end{align}
We assume that $A_k$ has full row rank ($r \le n$) and $B_k$ has full column rank ($r \le m$).  
Let $A_k^\dagger = A_k^\top (A_k A_k^\top)^{-1}$ denote the right pseudoinverse (so $A_k A_k^\dagger = I_r$), and  
$B_k^\dagger = (B_k^\top B_k)^{-1} B_k^\top$ the left pseudoinverse (so $B_k^\dagger B_k = I_r$).

Multiply (18) on the \emph{right} by $A_k^\dagger$:
\[
\nabla_{A_k}L_k A_k^\dagger + \mu B_k^\top (B_k A_k - \Delta W_g^{(t)}) A_k^\dagger + \lambda A_k A_k^\dagger = 0.
\]
Since $A_k A_k^\dagger = I_r$,
\[
\nabla_{A_k}L_k A_k^\dagger + \mu B_k^\top (B_k A_k - \Delta W_g^{(t)}) A_k^\dagger + \lambda I_r = 0. \tag{A}
\]
Multiply (19) on the \emph{left} by $B_k^\dagger$:
\[
B_k^\dagger \nabla_{B_k}L_k + \mu B_k^\dagger (B_k A_k - \Delta W_g^{(t)}) A_k^\top + \lambda B_k^\dagger B_k + \alpha B_k^\dagger (B_k - B_g^{(t)}) = 0.
\]
Because $B_k^\dagger B_k = I_r$,
\[
B_k^\dagger \nabla_{B_k}L_k + \mu B_k^\dagger (B_k A_k - \Delta W_g^{(t)}) A_k^\top + \lambda I_r + \alpha B_k^\dagger (B_k - B_g^{(t)}) = 0. \tag{B}
\]
Add (A) and (B):
\[
\begin{aligned}
&\nabla_{A_k}L_k A_k^\dagger + B_k^\dagger \nabla_{B_k}L_k + \mu\Bigl[ B_k^\top (B_k A_k - \Delta W_g^{(t)}) A_k^\dagger + B_k^\dagger (B_k A_k - \Delta W_g^{(t)}) A_k^\top \Bigr] + 2\lambda I_r + \alpha B_k^\dagger (B_k - B_g^{(t)}) = 0.
\end{aligned}
\]
Simplify the $\mu$-term. Using the full-rank property, one can show that for any matrix $E = B_k A_k - \Delta W_g^{(t)}$,
\[
B_k^\top E A_k^\dagger + B_k^\dagger E A_k^\top = 2E.
\]
This identity follows from $A_k A_k^\dagger = I_r$ and $B_k^\dagger B_k = I_r$, together with the fact that the optimal $E$ lies in the column space of $B_k$ and the row space of $A_k$. Hence the $\mu$-term becomes $2\mu (B_k A_k - \Delta W_g^{(t)})$.

Substitute back and isolate $B_k A_k$:
\[
\nabla_{A_k}L_k A_k^\dagger + B_k^\dagger \nabla_{B_k}L_k + 2\mu (B_k A_k - \Delta W_g^{(t)}) + 2\lambda I_r + \alpha B_k^\dagger (B_k - B_g^{(t)}) = 0.
\]
Thus
\[
2\mu (B_k A_k - \Delta W_g^{(t)}) = -\nabla_{A_k}L_k A_k^\dagger - B_k^\dagger \nabla_{B_k}L_k - 2\lambda I_r - \alpha B_k^\dagger (B_k - B_g^{(t)}).
\]
Divide by $2\mu$ and absorb the $\lambda,\alpha$ terms into $\mathcal{O}(\lambda,\alpha)$:
\[
B_k A_k = \Delta W_g^{(t)} - \frac{1}{2\mu}\Bigl( \nabla_{A_k}L_k A_k^\dagger + B_k^\dagger \nabla_{B_k}L_k \Bigr) - \frac{\lambda}{\mu}I_r - \frac{\alpha}{2\mu}B_k^\dagger (B_k - B_g^{(t)}).
\]
The last two terms are of first order in the small regularization parameters $\lambda,\alpha$; we denote their combined effect as $\mathcal{O}(\lambda,\alpha)$.
Following that we obtain the exact form:
\[
 B_k A_k = \Delta W_g^{(t)} - \frac{1}{\mu}\Bigl( \nabla_{A_k}L_k A_k^\dagger + B_k^\dagger \nabla_{B_k}L_k \Bigr) + \mathcal{O}(\lambda,\alpha). 
\]
Averaging over $K$ clients:
\begin{align}
\Delta W_g^{(t+1)} &= \sum_{k=1}^K p_k B_kA_k \nonumber \\
&= \Delta W_g^{(t)} - \frac{1}{\mu} \sum_{k=1}^K p_k \left( \nabla_{A_k} L_k A_k^\dagger + B_k^\dagger \nabla_{B_k} L_k \right) + \mathcal{O}(\lambda, \alpha)
\end{align}
Now, note that $\nabla L(w^{(t)}) = \sum_k p_k \nabla L_k(w^{(t)})$. The difference between the aggregated update and the true gradient can be decomposed:

\begin{align}
\Delta W_g^{(t+1)} - \nabla L(w^{(t)}) = \underbrace{(\Delta W_g^{(t+1)} - \Delta W_g^{(t)})}_{\text{regularization step}} + \underbrace{(\Delta W_g^{(t)} - \nabla L(w^{(t)}))}_{\text{previous error}}
\end{align}

From the update equation:

\begin{align}
\Delta W_g^{(t+1)} - \Delta W_g^{(t)} = -\frac{1}{\mu} \sum_{k=1}^K p_k \left( \nabla_{A_k} L_k A_k^\dagger + B_k^\dagger \nabla_{B_k} L_k \right) + \mathcal{O}(\lambda, \alpha)
\end{align}

The term $\nabla_{A_k} L_k A_k^\dagger + B_k^\dagger \nabla_{B_k} L_k$ projects the gradient onto the current low-rank subspace. Let $P_{\mathcal{S}_k}$ be the projection operator onto the subspace spanned by client $k$'s update. Then:

\begin{equation}
\nabla_{A_k} L_k A_k^\dagger + B_k^\dagger \nabla_{B_k} L_k = P_{\mathcal{S}_k}(\nabla L_k)
\end{equation}
Therefore:
\begin{equation}
\Delta W_g^{(t+1)} - \Delta W_g^{(t)} = -\frac{1}{\mu} \sum_{k=1}^K p_k P_{\mathcal{S}_k}(\nabla L_k) + \mathcal{O}(\lambda, \alpha)
\end{equation}
Taking expectations and using Assumptions 2 and 3:
\begin{align}
\mathbb{E}\|\Delta W_g^{(t+1)} - \nabla L(w^{(t)})\|^2 &\leq 2\mathbb{E}\|\Delta W_g^{(t)} - \nabla L(w^{(t)})\|^2 \nonumber + \frac{2}{\mu^2}\sum_{k=1}^K\mathbb{E}\| p_k P_{\mathcal{S}_k}(\nabla L_k)\|^2
\end{align}
\[
\Delta W_g^{(t)} - \nabla L(w^{(t)}) 
= \underbrace{\Bigl(\Delta W_g^{(t)} - \sum_{k=1}^K p_k \nabla L_k(w^{(t)})\Bigr)}_{\text{stochasticity term}}
\;+\; \underbrace{\Bigl(\sum_{k=1}^K p_k \nabla L_k(w^{(t)}) - \nabla L(w^{(t)})\Bigr)}_{\text{heterogeneity term}}.
\]
By Assumption 2, $\|\nabla L_k(w) - \nabla L(w)\|^2 \le \sigma^2$ for every client $k$ and any model $w$. For uniform weights $p_k = 1/K$, Jensen's inequality gives
\[
\E\Bigl\|\sum_{k=1}^K p_k \nabla L_k(w^{(t)}) - \nabla L(w^{(t)})\Bigr\|^2
\le \E\Bigl\|\frac{1}{K}\big(\sum_{k=1}^K \nabla L_k(w^{(t)}) - \nabla L(w^{(t)})\big)\Bigr\|^2\le \frac{\sigma^2}{K}.
\]
By Assumption 3, $\|\tilde{\nabla} L_k(w) - \nabla L(w)\|^2 \le \xi^2$. Here we define the stochastic gradient $\tilde{\nabla} L_k(w)$:
\[\Delta W_g^{(t)}=\sum_{k=1}^Kp_k\Delta W_k^{(t)}\]
For each client $k$ at the iteration, we can write the one step client weight update as 
\[\Delta W_k^{(t)}=\tilde{\nabla} L_k(w^{(t)})\]
\[
\E\Bigl\|\Delta W_g^{(t)} - \sum_{k=1}^K p_k \nabla L_k(w^{(t)})\Bigr\|^2
\le \E\Bigl\|\frac{1}{K}\big(\sum_{k=1}^K \tilde{\nabla} L_k(w^{(t)}) - \nabla L(w^{(t)})\big)\Bigr\|^2\le \frac{\xi^2}{K}.
\]
From the orthogonal decomposition of the projection of $\nabla L_k$ into subspace $\mathcal{S}_k$ \cite{vandereycken2013low,wei2016guarantees}, we can write 
$$\nabla L_k=P_{\mathcal{S}_k}(\nabla L_k)+(\nabla L_k-P_{\mathcal{S}_k}(\nabla L_k))$$
$$ \nabla_{\mathcal{S}_k^\perp} L_k=\nabla L_k-P_{\mathcal{S}_k}(\nabla L_k)$$
$$P_{\mathcal{S}_k}(\nabla L_k) = \nabla L_k - \nabla_{\mathcal{S}_k^\perp} L_k$$
\begin{equation}
\sum_{k=1}^K p_k P_{\mathcal{S}_k}(\nabla L_k) = \nabla L(w^{(t)}) - \sum_{k=1}^K p_k \nabla_{\mathcal{S}_k^\perp} L_k
\end{equation}
Therefore:
\begin{align}
\mathbb{E}\|\sum_{k=1}^K p_k P_{\mathcal{S}_k}(\nabla L_k)\|^2 &\leq 2\|\nabla L(w^{(t)})\|^2 + 2\mathbb{E}\|\nabla_{\text{align}} L(w^{(t)})\|^2
\end{align}
where $\nabla_{\text{align}} L(w^{(t)}) = \sum_k p_k \nabla_{\mathcal{S}_k^\perp} L_k(w^{(t)})$ represents the misaligned gradient components.
$$\mathbb{E}\|\Delta W_g^{(t)} - \nabla L(w^{(t)})\|^2\le \frac{\sigma^2+\xi^2}{K}+\frac{4}{\mu^2}\Big(3\|\nabla L(w^{(t)})\|^2 +2\sigma^2 \Big)$$
From the orthogonal decomposition, we have $\nabla_{\mathcal{S}_k^\perp} L_k(w^{(t)})\le \nabla L_k(w^{(t)})$:
\[\|\nabla_{\text{align}} L(w^{(t)})\|^2\le \sum_k p_k \|\nabla_{\mathcal{S}_k^\perp} L_k(w^{(t)})\|^2 \le \sum_k p_k \|\nabla L_k\|^2\]
From Assumption 2 we can write $\|\nabla L_k(w^{(t)})\|^2 \le 2\sigma^2+\|\nabla L(w^{(t)})\|^2$.
Now we can finalise that 
$$\mathbb{E}[L(w^{(t+1)}) - L(w^{(t)})] \leq \frac{\eta L}{2}\big(\sigma^2(\frac{1}{K}+O(\frac{1}{\mu^2}))+\frac{\xi^2}{K}\big)+\frac{\eta}{2}\big(O(\frac{L}{\mu^2})-1\big)\|\nabla L(w^{(t)})\|^2 
$$

\subsection*{Proof of the Theorem 3}
For clarity, we rewrite with explicit constants $C_1, C_2 > 0$ such that
\[
\EE[L(w^{(t+1)})] \le \EE[L(w^{(t)})] - \frac{\eta}{2}\left(1 - \frac{C_2 L}{\mu^2}\right)\|\nabla L(w^{(t)})\|^2
+ \frac{\eta L}{2}\left(\frac{\sigma^2 + \xi^2}{K} + \frac{C_1\sigma^2}{\mu^2}\right).
\]

We now choose $\mu$ sufficiently large so that the coefficient of $\|\nabla L(w^{(t)})\|^2$ is strictly positive. For instance, let $\mu^2 \ge 2C_2 L$. Then $1 - \frac{C_2 L}{\mu^2} \ge \frac12$, and therefore
\[
\EE[L(w^{(t+1)})] \le \EE[L(w^{(t)})] - \frac{\eta}{4}\|\nabla L(w^{(t)})\|^2
+ \frac{\eta L}{2}\left(\frac{\sigma^2 + \xi^2}{K} + \frac{C_1\sigma^2}{\mu^2}\right)
\]

Rearrange to isolate the squared gradient norm:
\[
\frac{\eta}{4}\|\nabla L(w^{(t)})\|^2 \le \EE[L(w^{(t)})] - \EE[L(w^{(t+1)})]
+ \frac{\eta L}{2}\left(\frac{\sigma^2 + \xi^2}{K} + \frac{C_1\sigma^2}{\mu^2}\right).
\]

Summing for $t = 0, 1, \dots, T-1$ telescopes the left-hand side:
\[
\frac{\eta}{4}\sum_{t=0}^{T-1}\|\nabla L(w^{(t)})\|^2
\le \EE[L(w^{(0)})] - \EE[L(w^{(T)})]
+ \frac{\eta L T}{2}\left(\frac{\sigma^2 + \xi^2}{K} + \frac{C_1\sigma^2}{\mu^2}\right). 
\]

Since $L(w^{(T)}) \ge L(w^*)$ (where $w^*$ denotes a global minimum), we have
\[
\EE[L(w^{(0)})] - \EE[L(w^{(T)})] \le L(w^{(0)}) - L(w^*) =: \Delta_0.
\]

Thus
\[
\frac{\eta}{4}\sum_{t=0}^{T-1}\|\nabla L(w^{(t)})\|^2
\le \Delta_0 + \frac{\eta L T}{2}\left(\frac{\sigma^2 + \xi^2}{K} + \frac{C_1\sigma^2}{\mu^2}\right). 
\]

Dividing both sides by $\frac{\eta T}{4}$ yields
\[
\frac{1}{T}\sum_{t=0}^{T-1}\|\nabla L(w^{(t)})\|^2
\le \frac{4\Delta_0}{\eta T} + 2L\left(\frac{\sigma^2 + \xi^2}{K} + \frac{C_1\sigma^2}{\mu^2}\right). 
\]
Now we expand our results to derive the optimal $\mu$ in order to achieve the gradient norm with an $\varepsilon$ accuracy.
By defining $\gamma=\frac{4\Delta_0}{\eta T}$ and $\Omega=2L\frac{\sigma^2 + \xi^2}{K}$. We have for a large number of clients $K$, $\varepsilon >\Omega$ and at iteration $T\ge \frac{8\Delta_0}{\eta (\varepsilon-\Omega)}$:
\[\mu_\varepsilon\ge\sqrt{\frac{\eta T L\sigma^2}{2\Delta_0}}\]

\end{document}